\documentclass{article}

\usepackage{arxiv}

\usepackage[utf8]{inputenc} 
\usepackage[T1]{fontenc}    
\usepackage{hyperref}       
\usepackage{url}            
\usepackage{booktabs}       
\usepackage{amsfonts}       
\usepackage{nicefrac}       
\usepackage{microtype}      
\usepackage{lipsum}
\usepackage{graphicx}
\graphicspath{ {./images/} }

\usepackage{xcolor}

\usepackage{enumitem}
\usepackage{mathtools}
\usepackage{amsthm}
\usepackage{bm}
\newtheorem{theorem}{Theorem}

\newtheorem{proposition}{Proposition}
\newtheorem{corollary}{Corollary}
\newtheorem{definition}{Definition}
\usepackage{comment}
\usepackage[normalem]{ulem}

\title{Distributional simplicity bias and effective convexity in energy-based models}

\author{
  Aurélien Decelle \\
  Escuela Técnica Superior de Ingenieros Industriales \\
  Universidad Politécnica de Madrid \\
  Calle de José Gutiérrez Abascal 2, Madrid 28006, Spain \\
  and \\
  Departamento de Física Teórica \\
  Universidad Complutense de Madrid \\
  Plaza de las Ciencias 1, Madrid 28040 \\
  \texttt{adecelle@ucm.es} \\
  \And
  Alfonso de Jesús~Navas Gómez \\
  Departamento de Física Teórica \\
  Universidad Complutense de Madrid \\
  Plaza de las Ciencias 1, Madrid 28040 \\
  \texttt{alfonn01@ucm.es} \\
  \And
  Beatriz Seoane \\
  Departamento de Física Teórica \& IPARCOS \\
  Universidad Complutense de Madrid \\
  Plaza de las Ciencias 1, Madrid 28040 \\
  \texttt{beseoane@ucm.es} \\
}

\begin{document}
\maketitle
\begin{abstract}
Energy-based learning is a powerful framework for generative modelling, but its training is inherently non-convex, leading potentially to sensitivity to initialisation, poor local optima, and unstable gradient dynamics. We present a dynamical analysis of energy-based learning through the lens of the effective model, which can be interpreted as either a generalised Ising model with higher-order interactions or the Fourier expansion of the energy. Under sufficient expressivity, we show that the gradient flow induced by learning strictly positive distributions over binary variables admits two types of fixed points: data-consistent points, which exactly reproduce the target distribution, and spurious points, which satisfy stationarity without matching the target distribution. Around data-consistent points, we show that perturbations are either stable or neutral, with neutral directions leaving the effective model invariant. 
Finally, we show that gradient dynamics induce a hierarchy in which lower-order interactions are learned before higher-order ones. This provides a mechanistic explanation for the distributional simplicity bias, and clarifies why fixed points that are not data-consistent at low orders are not observed in practice.
\end{abstract}

\section{Introduction}
The strong generalisation performance of over-parametrised neural networks is widely attributed to implicit inductive biases, among which \textit{simplicity bias} plays a central role. This principle states that, when trained with stochastic gradient-based methods, models tend to learn simple functions first and only gradually fit more complex nonlinear structures~\cite{valle2018deep,kalimeris2019sgd,refinetti2023neural}. 
Recent empirical evidence, particularly in modern architectures such as transformers trained on natural language~\cite{rende2024distributional,belrose2024neural,garnier2024transformers}, suggests that this phenomenon extends beyond function approximation to the level of data distributions. In this setting, learning dynamics exhibit a structured progression in the order of interactions: models first reproduce low-order statistics and subsequently incorporate higher-order dependencies, corresponding to increasingly complex many-body interactions among input variables. 
This behaviour appears to be robust across a wide range of architectures and learning paradigms, including supervised models as well as unsupervised and generative frameworks such as autoencoders~\cite{kogler2024compression}, energy-based models~\cite{decelle2024inferring,decelle2025inferring}, and diffusion models~\cite{favero2025compositional,bardone2026theory}. Despite this extensive empirical support, a theoretical understanding of these phenomena remains limited and has so far been rigorously established only in simplified learning regimes~\cite{saad1995exact,mei2018mean,bardone2026theory}.

In this work, we show that distributional simplicity bias can be rigorously established for the broad family of energy-based models (EBMs) trained via log-likelihood maximisation. In this framework, probability distributions over discrete configurations $\{{\bm s}\}$ are defined through an energy function $E_{\bm \theta}({\bm s})$, parametrised by a feed-forward neural network, as
$$
P_{\bm \theta}({\bm s}) = \frac{1}{Z_{\bm \theta}} e^{-H_{\bm \theta}({\bm s})}, 
\qquad 
Z_{\bm \theta} = \sum_{{\bm s}} e^{-H_{\bm \theta}({\bm s})},
$$
where $Z_{\bm \theta}$ is the partition function ensuring normalisation.

Since the primary objective of EBMs is to capture the interaction structure among variables through the energy assigned to each configuration, they provide a natural framework to analyse learning dynamics in terms of many-body interactions. In particular, it has been recently shown~\cite{decelle2025inferring} how the energy function can be decomposed into contributions of increasing interaction order,
\begin{align}
    H_{\boldsymbol{\theta}}(\boldsymbol{s}) 
    &= - \sum_{i=1}^N h_i(\boldsymbol{\theta}) s_i 
       - \sum_{i_1 < i_2} J_{i_1 i_2}^{(2)}(\boldsymbol{\theta}) s_{i_1} s_{i_2} 
       - \sum_{i_1 < i_2 < i_3} J_{i_1 i_2 i_3}^{(3)}(\boldsymbol{\theta})  s_{i_1} s_{i_2} s_{i_3} 
       + \dots \nonumber,
\end{align}
where the effective fields $h_i$ and interaction coefficients $J_{i_1 \dots i_n}^{(n)}$ are explicit functions of the model parameters $\bm \theta$. This representation enables a reinterpretation of the EBM as an effective spin model with many-body interactions in the sense of statistical physics. Within this framework, learning can be understood as the progressive incorporation of higher-order interactions into the energy function. We show that gradient-based training induces a hierarchy in which lower-order interactions are learned first and on faster timescales than higher-order ones.

This work shows that, since the effective multi-body spin representation associated with a given data distribution is unique, the learning dynamics of EBMs, which are generally non-convex in their native parametrisation, become locally effectively convex when projected onto the space of effective interaction parameters. In this representation, non-linear EBMs admit data-consistent fixed points, up to flat directions that leave the effective multi-body spin model invariant. Although additional non-data-consistent stationary points may still exist, the hierarchical structure induced by maximum-likelihood gradient dynamics across interaction orders naturally suppresses many spurious solutions. In particular, gradient descent first aligns the model’s low-order moments with those of the data, thereby steering the dynamics toward the data-consistent fixed point.

Importantly, this interaction-based decomposition is not only conceptually meaningful for describing distributional simplicity bias but also operational: it provides a systematic framework for modelling complex systems in terms of effective spin variables, whose physical interaction network can be quantitatively inferred. The effectiveness of this inverse approach has been demonstrated in controlled synthetic settings—where the underlying generative parameters can be accurately recovered~\cite{decelle2024inferring,decelle2025inferring}—as well as in real-world applications, including protein sequence data~\cite{tubiana2019learning,decelle2025inferring}, where structural contacts can be inferred, and neuroscience~\cite{bereux2026uncovering}, where functional connectivity between neurons can be reconstructed.

\textbf{Main results}-- We develop a theoretical framework to analyse the learning dynamics of non-linear EBMs by mapping the energy function to a fully visible Boltzmann machine with arbitrarily high order interactions. Our main contributions are summarised as follows:
\begin{itemize}
    \item We analyse the learning dynamics of EBMs and establish the existence of a marginally stable fixed point at which the model exactly reproduces all the correlations of the target distribution. We further show that this fixed point can be stabilised through an $\ell_2$ regularisation in the effective-coupling space.

    \item We characterise the temporal evolution of the effective couplings during learning. In the large-parameter regime, and under mild assumptions on the sparsity and parametrisation of the ground-truth distribution, we show that the learning dynamics are sequentially ordered with respect to the interaction order $n$: effective couplings associated with lower-order interactions are learned faster than those of higher order. We refer to this phenomenon as the \emph{Distributional Simplicity Bias} (DSB).

\end{itemize}
\textbf{Related works}-- 
Understanding why gradient-based optimisation performs efficiently in highly non-convex machine-learning problems has become a central theoretical question over the last decade. Most existing works, however, focus on supervised or discriminative settings, where the objective function can be evaluated directly, and the optimisation dynamics are decoupled from sampling issues. Early studies, inspired by spin-glass theory, argued that in high-dimensional parameter spaces optimisation difficulties are dominated by saddle points and flat directions rather than by isolated poor local minima~\cite{Choromanska2015}. Subsequent works showed that the stochasticity of SGD naturally destabilises saddle points and guides the dynamics toward broader and more favourable minima~\cite{ge2015escaping,jin2021nonconvex}. More recently, overparameterisation and implicit regularisation have been proposed as key mechanisms explaining why gradient-based optimisation remains tractable despite the underlying non-convexity~\cite{jacot2018neural,mei2019mean,ma2018implicit,zhang2016understanding}. 

In contrast, the theoretical understanding of non-convex optimisation in EBMs remains comparatively limited, despite recent attempts to convexify the learning dynamics of Restricted Boltzmann Machines~\cite{decelle2021exact}. Furthermore, in these models, optimisation is intrinsically coupled to approximate Monte Carlo sampling and to the estimation of an intractable partition function, introducing additional sources of metastability, slow mixing, and dynamical trapping during learning~\cite{Nijkamp2020,Decelle2021,agoritsas23a}. From the perspective of statistical physics, these phenomena are naturally associated with rugged free-energy landscapes, glassy dynamics, and competing metastable states, highlighting deep connections between the learning dynamics of EBMs and disordered systems.

\section{Fully visible Boltzmann machine with higher-order interactions}
\label{fully_visible_Boltzmann_machine_section}
Consider the energy function of a Boltzmann machine with arbitrary high-order interactions
\begin{align}
    H_{\boldsymbol{\phi}}(\boldsymbol{s}) 
    &=  - \sum_{i=1}^N h_i s_i - \sum_{n=2}^N \sum_{i_1 < \dots < i_n} J_{i_1 \dots i_n}^{(n)} \prod_{\nu=1}^{n} s_{i_\nu}, 
    \label{generalized_ising_model_main}
\end{align}
such that the probability mass function of our model is given by a Boltzmann distribution, \( p_{\mathrm{model}}  (\boldsymbol{s}) \propto e^{-H_{\boldsymbol{\phi}}(\boldsymbol{s})}.\)
Here, we use \( \boldsymbol{\phi} \coloneq (\boldsymbol{h}, \boldsymbol{J} )\) to denote the vector of all coupling parameters. In the following, we will use \( \phi_{ I }^{(1)} \coloneq h_i, \) with \( I= \{i\}, \) and \( \phi_{ I }^{(n)} \coloneq J_{i_1 \dots i_n}^{(n)}, \) with \( I \coloneq \{i_1, \dots, i_n \}\), i.e. $n=|I|$, to denote effective fields and $n$-body couplings parameters.  This model could, in principle, encode any strictly positive probability distribution \(p_\mathrm{data} :\{-1, 1\}^N \to (0,1) \). We formalise this statement in the following theorem:

\begin{theorem}
    \textbf{Universal approximation theorem of the fully visible Boltzmann machine with higher-order interaction}. Any probability mass function \( p_\mathrm{data} \colon \{-1, 1 \}^N \to (0,1) \) can be parametrized in terms of an energy \( H_{\boldsymbol{\phi}}\colon \{-1, 1 \}^N \to \mathbb{R} \) with the functional form given by Eq.~\eqref{generalized_ising_model_main}, such that \(p_{\mathrm{data}}(\boldsymbol{s}) \propto e^{-H_{\boldsymbol{\phi}} (\boldsymbol{s})}\), \( \forall \boldsymbol{s} \in \{-1, 1 \}^N.\)
\end{theorem}

The proof is straightforward from the Fourier Expansion Theorem for pseudo-Boolean functions (see Corollary~\ref{Fourier_expansion_corollary} in Appendix~\ref{pseudo-Boolean_analysis_section}), as soon as $p_{\mathrm{data}}(\boldsymbol{s}) > 0$ for all $\boldsymbol{s} \in \{-1,1\}^N$.

The higher-order Boltzmann machine can be trained by minimising the Kullback-Leibler divergence or, equivalently, the negative log-likelihood with respect to the empirical distribution of the dataset. Such a negative log-likelihood function is given by
\begin{equation}
    -\mathcal{L} = -\sum_{\boldsymbol{s}\in \{-1, 1 \}^N } p_{\mathrm{data}}(\boldsymbol{s}) \left( -H_{\boldsymbol{\phi}}(\boldsymbol{s}) \right) + Z_{\boldsymbol{\phi}}.
    \label{generalized_boltzmann_log-likelihood}
\end{equation}
\begin{theorem} \label{thm:convex1}
\textbf{Convexity of higher-order Boltzmann Machine Learning}. The negative log-likelihood function of a higher-order Boltzmann machine is convex.
\end{theorem}
The convexity of the log-likelihood follows directly from the exponential form of the distribution, i.e. from the fact that the energy is linear in the parameters. In this case, the Hessian of the negative log-likelihood can be written as the covariance matrix of the conjugate observables,
\begin{equation}
    \nabla^2_{\boldsymbol{\phi}} (-\mathcal{L})
    =
    \mathrm{Cov}_{\mathrm{model}}(\mathbf{O}),
    \label{Hessian_Log-likelihood}
\end{equation}
where
\[
\mathbf{O}
=
\left(
\prod_{i \in I} s_i
\right)_{I \subseteq [N]}
=
\left(
s_1,\dots,s_N,
s_1s_2,\dots,
\prod_{k=1}^N s_k
\right)
\]
denotes the set of all distinct monomials generated by the variables $s_i$. Since these observables are linearly independent, $\mathrm{Cov}_{\mathrm{model}}(\mathbf{O})$ is strictly positive definite, implying that $-\mathcal{L}_{\boldsymbol{\phi}}$ is strictly convex. Consequently, under standard regularity assumptions, gradient-based optimisation converges to a unique global minimum $\boldsymbol{\phi}^\ast$, where exact moment matching holds
\begin{equation}
    \left\langle \mathbf{O} \right\rangle_\mathrm{data}
    =
    \left\langle \mathbf{O} \right\rangle_\mathrm{model}.
    \label{data-consistent_condition}
\end{equation}

\section{Non-linear energy-based models}
\label{Non-linear_EBM_section}

Since the number of parameters in Eq.~\eqref{generalized_ising_model_main} grows as $\sim 2^{N}$, directly learning such a generalised Boltzmann machine becomes infeasible. Nevertheless, architectures such as Restricted Boltzmann Machines (RBMs), which are universal approximators of discrete probability distributions~\cite{le2008representational,montufar2011refinements,montufar2011expressive}, can implicitly reproduce arbitrarily high-order statistical dependencies without introducing an explicit parameter for each interaction term. Similar phenomenology is also observed empirically in more general neural energy-based models.

In the most general setting, one defines an energy function $E_{\boldsymbol{\theta}}:\{-1,1\}^N \to \mathbb{R}$, differentiable with respect to parameters $\boldsymbol{\theta}$. The associated model distribution is the Boltzmann distribution
 \(p_{\mathrm{model}}(\boldsymbol{s}) \propto e^{-E_{\boldsymbol{\theta}} (\boldsymbol{s})}\). The  log-likelihood then reads
\begin{equation}
   \mathcal{L} = \sum_{\boldsymbol{s}\in \{-1, 1 \}^N } p_{\mathrm{data}}(\boldsymbol{s}) \left( -E_{\boldsymbol{\theta}}  ( \boldsymbol{s}) \right) + \ln Z_{\bm{\theta}}, 
\end{equation}
and the analogous gradient flow in this case is given by
\begin{equation}
    \dot{\boldsymbol{\theta}} =  \nabla_{\boldsymbol{\theta}} \mathcal{L},
    \label{parameters_gradient}
\end{equation}
for which convexity can no longer be established in general, in contrast with Theorem~\ref{thm:convex1}.

In order to analyse the convexity properties of the EBM, we rewrite the Hamiltonian in the functional form of Eq.~\eqref{generalized_ising_model_main}, where each effective coupling \( \phi_{I}^{(n)} \) is expressed as a function of the model parameters \(\boldsymbol{\theta}\) (see Corollary~\ref{Fourier_expansion_corollary} in Appendix~\ref{pseudo-Boolean_analysis_section}). Equation~\eqref{parameters_gradient} can then be rewritten as
\begin{equation}
    \dot{\boldsymbol{\theta}} =
    \left( \frac{\partial \boldsymbol{\phi}}{\partial \boldsymbol{\theta}} \right)^\top
    \nabla_{\boldsymbol{\phi}} \mathcal{L}.
    \label{parameters_gradient_2}
\end{equation}
The fixed points of the learning dynamics therefore satisfy
\begin{equation}
    \left( \frac{\partial \boldsymbol{\phi}}{\partial \boldsymbol{\theta}} \right)^\top
    \nabla_{\boldsymbol{\phi}} \mathcal{L}=0.
\end{equation}
This includes data-consistent fixed points $\boldsymbol{\theta}^{\ast}$, for which $\nabla_{\boldsymbol{\phi}}\mathcal{L}=0$ and Eq.~\eqref{data-consistent_condition} is satisfied. It may also include spurious non-data-consistent fixed points \( \boldsymbol{\theta}^\dagger \), arising whenever $\nabla_{\boldsymbol{\phi}}\mathcal{L}$ belongs to the null space of $\left(\partial \boldsymbol{\phi}/\partial \boldsymbol{\theta}\right)^\top$.
This reflects a geometric constraint induced by the parametrisation $\boldsymbol{\phi}(\boldsymbol{\theta})$, which restricts the admissible directions in $\boldsymbol{\phi}$-space. If the gradient $\nabla_{\boldsymbol{\phi}}\mathcal{L}$ points along a direction orthogonal to the image of the Jacobian, it cannot be realised by any infinitesimal variation of $\boldsymbol{\theta}$. Consequently, the dynamics stall even though the model has not reached a likelihood optimum.

A sufficient condition for the existence of such spurious fixed points is that the Jacobian vanishes,
\begin{equation}
    \left( \frac{\partial \boldsymbol{\phi}}{\partial \boldsymbol{\theta} } \right) = 0.
    \label{spurious_point_condition}
\end{equation}
Using the orthogonality of the Fourier basis, this condition is equivalent to
\begin{align}
    \sum_a \sum_{\substack{I \subseteq [N], \\ |I| \ge 1}}
    \left( \frac{\partial \phi_{I}^{(n)}}{\partial \theta_a} \right)^2
    =
    \sum_a
    \mathrm{Var}_{\boldsymbol{s} \sim \{-1,1\}^N}
    \left[
    \frac{\partial E_{\boldsymbol{\theta}}(\boldsymbol{s})}{\partial \theta_a}
    \right]
    =0.
    \label{spurious_condition_variance}
\end{align}
In practice, this condition typically requires all parameters to vanish in standard EBM architectures, since the derivatives of the energy with respect to the parameters generally depend explicitly on the spin configuration. For instance, in the RBM,
\begin{equation}
    E_{\boldsymbol{\theta}} (\boldsymbol{s})
    =
    \sum_a
    \ln \cosh
    \left(
    \sum_i w_{ia} s_i + \zeta_a
    \right)
    +
    \sum_i \eta_i s_i,
\end{equation}
The condition~\eqref{spurious_point_condition} cannot be strictly satisfied because $\partial \phi_i^{(1)} / \partial \eta_i = 1$. Nevertheless, we find a spurious fixed point $\boldsymbol{\theta}^\dagger$ when the visible fields are adjusted to match the local magnetisations, namely $\eta_i = \operatorname{arctanh}\langle s_i\rangle$, while the remaining parameters ${\boldsymbol{w},\boldsymbol{\zeta}}$ are set to zero.

We now present our first result, regarding the effective stability of data-consistent fixed points:
\begin{theorem}
\label{stability_of_fixed_point_theorem}
    \textbf{Effective stability of data-consistent fixed points}. Let \(E_{\boldsymbol{\theta}}\) be a differentiable EBM with parameters \( \boldsymbol{\theta}\), and let \( \boldsymbol{\theta}^{\ast}\) be {a data-consistent} fixed point in the gradient flow induced dynamics given in Eq.~\eqref{parameters_gradient}. Then, the effective model  \(\boldsymbol{\phi}(\boldsymbol{\theta}^{\ast})\) is locally stable.
\end{theorem}
\begin{proof}
Local stability of a fixed-point \( \boldsymbol{\theta}^\ast \) under the dynamics induced by the gradient flow given in Eq.~\eqref{parameters_gradient} would require positive definiteness of the Hessian
 \begin{align}
     \nabla^2_{\boldsymbol{\theta}} (-\mathcal{L}) = \left( \frac{\partial^2 \boldsymbol{\phi}}{\partial \boldsymbol{\theta}^2}\right)^\top \nabla_{\phi} (-\mathcal{L}) +  \left( \frac{\partial \boldsymbol{\phi}}{\partial \boldsymbol{\theta} }  \right)^\top \nabla_{\boldsymbol{\phi}}^2 (-\mathcal{L}) \left( \frac{\partial \boldsymbol{\phi}}{\partial \boldsymbol{\theta} } \right)
     \label{log_likelihood_hessian}
 \end{align}
evaluated at \( \boldsymbol{\theta} = \boldsymbol{\theta}^\ast  \).  Replacing \(\boldsymbol{\theta}^\ast\) in Eq.~\eqref{log_likelihood_hessian} causes the first term on the r.h.s. to vanish, since the condition in Eq.~\eqref{data-consistent_condition} is equivalent to \(\nabla_{\phi} (-\mathcal{L}) = 0 \). Then, we have
\begin{equation}
    \nabla_{\boldsymbol{\theta}}^2 (-\mathcal{L}) \Big|_{\boldsymbol{\theta} = \boldsymbol{\theta}^\ast} \! = \left. \left( \frac{\partial \boldsymbol{\phi}}{\partial \boldsymbol{\theta}} \right)^\top \mathrm{Cov}_ \mathrm{model}(\mathbf{O}) \left( \frac{\partial \boldsymbol{\phi}}{\partial \boldsymbol{\theta}} \right) \right|_{\boldsymbol{\theta} = \boldsymbol{\theta}^\ast}  ,
    \label{parameters_hessian_evaluated}
\end{equation}
We can verify that the r.h.s. of the above expression is positive semi-definite. This implies marginal stability of \(\boldsymbol{\theta}^\ast \). In other words, perturbations around the fixed point are either stable, in the sense that they return to the same fixed point \(\boldsymbol{\theta}^\ast\), or neutral, i.e., neither attracted nor repelled by the fixed point. Since $\mathrm{Cov}_ \mathrm{model}(\mathbf{O})$ is positive definite, the r.h.s. of Eq. \eqref{parameters_hessian_evaluated} is not positive-definite if and only if there are some (neutral) perturbations $\boldsymbol{\theta} = \boldsymbol{\theta}^\ast + \delta \boldsymbol{\theta}$ such that
\begin{equation}
    \left. \frac{\partial \boldsymbol{\phi}}{\partial \boldsymbol{\theta}} \right|_{\boldsymbol{\theta} = \boldsymbol{\theta}^\ast} \delta \boldsymbol{\theta} = 0,
\end{equation}
Since the linear expansion of the effective model evaluated at $\boldsymbol{\theta} = \boldsymbol{\theta}^\ast + \delta \boldsymbol{\theta}$ around $\boldsymbol{\theta}^\ast$ reads
\begin{equation}
     \boldsymbol{\phi} (\boldsymbol{\theta}^\ast + \delta \boldsymbol{\theta}) = \boldsymbol{\phi}(\boldsymbol{\theta}^\ast) + \left. \frac{\partial \boldsymbol{\phi}}{\partial \boldsymbol{\theta}} \right|_{\boldsymbol{\theta} = \boldsymbol{\theta}^\ast} \delta \boldsymbol{\theta} = \boldsymbol{\phi}(\boldsymbol{\theta}^\ast).
\end{equation}
Then, we conclude that the effective model \( \boldsymbol{\phi}(\boldsymbol{\theta}^\ast) \) is locally stable: under gradient descent, perturbations along stable directions of the Hessian given in Eq. \eqref{parameters_hessian_evaluated} decay back to \( \boldsymbol{\theta}^\ast \), while perturbations along neutral directions leave \( \boldsymbol{\phi}(\boldsymbol{\theta}^\ast) \) unchanged.
\end{proof}

A direct consequence of the effective stability theorem for non-linear EBMs is that the attractors of the learning dynamics under gradient ascent need not be isolated fixed points in parameter space, but may instead form a degenerate manifold of fixed points, with every point on this manifold corresponding to the same effective model that reproduces the data statistics.

\section{$\ell_2$-regularization}

In sections \ref{fully_visible_Boltzmann_machine_section}, \ref{Non-linear_EBM_section}, we have assumed a strictly positive probability mass function \(p_{\mathrm{data}}\). This is a strong assumption, since in realistic scenarios we only have access to a finite dataset, leading to an empirical distribution of the form
\[ p_\mathrm{data} (\boldsymbol{s}) = \frac{1}{M}\sum_{d=1}^M \delta(\boldsymbol{s}-\boldsymbol{s}^{(d)}). \]
This distribution is supported only on the observed samples and vanishes elsewhere, i.e., $p_{\mathrm{data}}(\boldsymbol{s}) = 0$ for almost every \(\boldsymbol{s} \in \{-1, 1 \}^N \). As a consequence, likelihood-based training may drive the model to assign vanishing probability mass to unseen configurations, leading to divergent effective couplings in highly expressive models. We can avoid this behaviour by introducing an \(\ell_2\)-penalty in the negative log-likelihood, 
\begin{equation}
    - \mathcal{L}_{\mathrm{ridge}} = -\sum_{\boldsymbol{s}\in \{-1, 1 \}^N } p_{\mathrm{data}}(\boldsymbol{s}) \left( -E_{\boldsymbol{\theta}}  ( \boldsymbol{s}) \right) + \ln Z_{\bm{\theta}} + \frac{\lambda}{2} \sum_{\substack{I \subseteq [N], \\ |I|\ge 1}} \left( \phi_{I}^{(n)} \right)^2. 
\end{equation}
We note that the gradient flow dynamics, given by
\begin{equation}
    \dot{\boldsymbol{\theta}} = \left( \frac{\partial \boldsymbol{\phi}}{\partial \boldsymbol{\theta}} \right)^\top \nabla_{\boldsymbol{\phi}}  \mathcal{L}_{\mathrm{ridge}},
\end{equation}
also have data-consistent fixed-points \(\boldsymbol{\theta}_{\mathrm{ridge}}^\ast \) and spurious fixed-points \(\boldsymbol{\theta}_{\mathrm{ridge}}^\dagger \). The data-consistent fixed-point condition is modified as
\begin{equation}
    \left\langle \boldsymbol{O} \right\rangle_\mathrm{data} = \left\langle \boldsymbol{O} \right\rangle_\mathrm{model} + {\lambda} \boldsymbol{\phi}.
\end{equation}

Similarly, we compute the Hessian of the ridge negative log-likelihood evaluated in \(\boldsymbol{\theta} = \boldsymbol{\theta}_{\mathrm{ridge}}^\ast \),
\begin{equation}
    \nabla_{\boldsymbol{\theta}}^2 (-\mathcal{L}_{\mathrm{ridge}}) \Big|_{\boldsymbol{\theta} = \boldsymbol{\theta}_{\mathrm{ridge}}^\ast} \! = \left. \left( \frac{\partial \boldsymbol{\phi}}{\partial \boldsymbol{\theta}} \right)^\top \left[ \mathrm{Cov}_ \mathrm{model}(\mathbf{O}) + {\lambda}\mathrm{I} \right] \left( \frac{\partial \boldsymbol{\phi}}{\partial \boldsymbol{\theta}} \right) \right|_{\boldsymbol{\theta} = \boldsymbol{\theta}_{\mathrm{ridge}}^\ast},
\end{equation}
from which it follows that the theorem of effective stability of data-consistent fixed points, Theorem~\ref{stability_of_fixed_point_theorem}, also holds in this case. Since the exact computation of \(\sum_{I \subseteq [N], \ |I|\ge 1} \left( \phi_{I}^{(n)} \right)^2\) is computationally prohibitive, an estimator of this quantity is required for practical implementations. Thus, from Parseval's theorem of pseudo-Boolean functions (Corollary~\ref{Parseval_theorem} in appendix~\ref{pseudo-Boolean_analysis_section}), we derive the following relationship:
\begin{align}
    \sum_{\substack{I \subseteq [N], \\ |I|\ge 1}} \left( \phi_{I}^{(n)} \right)^2 
    &= \mathrm{Var}_{\boldsymbol{s} \sim \{-1, 1 \}^N} \left[ E_{\boldsymbol{\theta}} (\boldsymbol{s}) \right],
    \label{ridge_regression_contraction}
\end{align}
which we can use to implement a stochastic estimator using random samples from the uniform distribution. Finally, we mention that using the typical $\ell_2$ regularization on the parameters $\boldsymbol{\theta}$, i.e, 
$+ \sum_a \theta_a^2,$ does not guarantee effective stability of the data consistent fixed-points

\section{Distributional Simplicity Bias (DSB)}
To understand how DSB naturally arises in energy-based learning, we write the equations that describe the dynamics of effective couplings, 
\begin{align}
    \dot{\boldsymbol{\phi}} = \left( \frac{\partial \boldsymbol{\phi} }{\partial \boldsymbol{\theta} }\right) \dot{\boldsymbol{\theta}} = \left( \frac{\partial \boldsymbol{\phi} }{\partial \boldsymbol{\theta} } \right) \left( \frac{\partial \boldsymbol{\phi} }{\partial \boldsymbol{\theta} } \right)^\top \nabla_{\boldsymbol{\phi}} \mathcal{L}.
    \label{effective_couplings_dynamics}
\end{align}
From Eq.~\eqref{effective_couplings_dynamics}, we can write the expression for a single effective coupling
\begin{align}
    &\dot{\phi}_{I}^{(n)} 
   \textstyle = \left[ \nabla_{\boldsymbol{\theta}} {\phi}_{I}^{(n)} \right]^\top  \left( \frac{\partial \boldsymbol{\phi} }{\partial \boldsymbol{\theta} } \right)^\top \nabla_{\boldsymbol{\phi}} \mathcal{L} \nonumber \\
    & \quad\textstyle = \sum_{\substack{J \subseteq [N], \\ |J| \ge 1}} \left( \nabla_{\boldsymbol{\theta}} {\phi}_{I}^{(n)} \cdot \nabla_{\boldsymbol{\theta}} {\phi}_{J}^{(m)} \right) \left[  \left\langle \prod_{j\in J}^m s_{j} \right\rangle_\mathrm{data}  - \left\langle \prod_{j\in J} s_{j} \right\rangle_\mathrm{model} \right] \nonumber \\
    & \quad\textstyle = \left\| \nabla_{\boldsymbol{\theta}} {\phi}_{I} \right\| \sum_{\substack{J \subseteq [N], \\ |J| \ge 1}} \left\| \nabla_{\boldsymbol{\theta}} {\phi}_{J}^{(m)} \right\| \cos \angle \left( \nabla_{\boldsymbol{\theta}} \phi_{I}^{(n)}, \nabla_{\boldsymbol{\theta}} \phi_{J}^{(m)} \right) \left[  \left\langle \prod_{j \in J} s_{j} \right\rangle_\mathrm{data} - \left\langle \prod_{j \in J} s_{j} \right\rangle_\mathrm{model} \right]
    \label{effective_couplings_dynamics_3}.
\end{align}
Assuming that $\nabla_{\boldsymbol{\theta}} \phi_{I}^{(n)}$ and $\nabla_{\boldsymbol{\theta}} \phi_{J}^{(m)}$ are isotropic random vectors in a high-dimensional space, the cosine term in Eq.~\eqref{effective_couplings_dynamics_3} can be approximated, for \( I \neq J \), as
\begin{equation}
    \cos \angle\!\left( \nabla_{\boldsymbol{\theta}} \phi_{I}^{(n)}, \nabla_{\boldsymbol{\theta}} \phi_{J}^{(m)} \right) \sim O \left( N_{\boldsymbol{\theta}}^{-1/2} \right).
    \label{high-dimension_approx}
\end{equation}
Then, using Eq.~\eqref{high-dimension_approx} in Eq.~\eqref{effective_couplings_dynamics_3}, we obtain 
\begin{align}
    \dot{\phi}_{I}^{(n)}  
    &= \textstyle\left\| \nabla_{\boldsymbol{\theta}} {\phi}_{I}^{(n)} \right\| 
    \Bigg[ \left\| \nabla_{\boldsymbol{\theta}} {\phi}_{I}^{(n)} \right\| 
    \left(\left\langle \prod_{i \in I} s_{i} \right\rangle_\mathrm{data} - 
    \left\langle \prod_{i \in I} s_{i} \right\rangle_\mathrm{model} \right) 
    \nonumber \\
    & \textstyle\qquad + O \left( \frac{1}{\sqrt{N_{\boldsymbol{\theta}}}}  \sum_{\substack{J \subseteq [N], \\ |J| \ge 1 \\ J \neq I}}  \left\| \nabla_{\boldsymbol{\theta}} {\phi}_{J}^{(m)} \right\| 
    \left( \left\langle \prod_{j \in J} s_{j} \right\rangle_\mathrm{data} - \left\langle \prod_{j \in J} s_{j} \right\rangle_\mathrm{model} \right) \right) \Bigg].
    \label{effective_couplings_dynamics_4}
\end{align}

From the previous expression, we observe that in the large  $N_{\boldsymbol{\theta}}$ limit, the training dynamics are well approximated by
\begin{equation}
    \dot{\phi}_{I}^{(n)} \approx 
    \left\| \nabla_{\boldsymbol{\theta}} {\phi}_{I}^{(n)} \right\|^2 
    \left[ \left\langle \prod_{i \in I} s_{i} \right\rangle_\mathrm{data}
    - \left\langle \prod_{i \in I} s_{i} \right\rangle_\mathrm{model}
    \right].
    \label{effective_couplings_dynamics_5}
\end{equation}
To be more precise with the kind of limit we are taking in Eq. \eqref{effective_couplings_dynamics_5}, we assume a sparse ground-truth distribution, in the sense that the number of relevant couplings (those with magnitude larger than $\epsilon$) is much smaller than the number of parameters in the model $N_{\boldsymbol{\theta}}$, which itself is far below the $\sim 2^N - 1$ parameters required to specify a fully general Boltzmann machine with interactions up to order $N$.:
$$
\textstyle 2^{N}-1 \gg N_{\boldsymbol{\theta}} \gg \sum_{I \subseteq [N]}\boldsymbol{1}_{\left\{ |\phi_{I}^{(n)}| > \epsilon \right\}}.
$$
Because of Eqs.~\eqref{effective_couplings_dynamics_4} and~\eqref{effective_couplings_dynamics_5}, we argue that DBS in energy-based learning emerges as a consequence of the decay of \(\left\| \nabla_{\boldsymbol{\theta}} {\phi}_{I}^{(n)} \right\|^2 \)  with the order \(n\). Indeed, the \textit{Low-degree Spectral Concentration Theorem} (see Theorem~\ref{low-degree_concentration_theorem} in appendix~\ref{pseudo-Boolean_analysis_section} ) for pseudo-boolean functions allows us to establish the following:
\begin{proposition}
\label{bound_proposition}
\textbf{Upper bound on high-order dynamical terms}. Let \(E_{\boldsymbol{\theta}}\) be a function differentiable EBM with parameters \(\boldsymbol{\theta}\), and let \( \{\phi_I^{(n)} \}_{|I| \ge n} \) be the set of effective couplings of order at least \(n\). Then,
\begin{equation}
   \sum_{\substack{I \subseteq [N], \\ |I| \ge n}} \left\| \nabla_{\boldsymbol{\theta}} {\phi}_{I}^{(n)} \right\|^2 = \sum_a \sum_{\substack{I \subseteq [N], \\ |I| \ge n}} \left(\frac{\partial \phi_I^{(n)}}{\partial \theta_a } \right)^2 \le \frac{1}{n} \sum_{a} \mathbf{I} \left( \frac{\partial E_{\boldsymbol{\theta}}}{\partial \theta_a} \right).\label{eq:bound}
\end{equation}
 Here, \(\mathbf{I}\) is the \textit{total influence} functional for pseudo-boolean functions (see Definition~\ref{influence_definition} in appendix~\ref{pseudo-Boolean_analysis_section}).  
\end{proposition}

Here, we have used the following definition for the influence:
\begin{definition}
\label{influence_definition}
    The \textbf{influence} of coordinate \( i \) on \( f \colon \{-1, 1 \}^N \to \mathbb{R} \) is defined to be
    \[ \mathbf{Inf}_i[f] \coloneq \frac{1}{2^{N}} \!\!\!\! \sum_{\boldsymbol{s'} \in \{-1, 1\}}^N \left[ \frac{f\left( \boldsymbol{s}^{(i \to 1)} \right) - f \left(\boldsymbol{s}^{(i \to -1)} \right)}{2}   \right]^2. \]
    Here we have used the notation $\boldsymbol{s}^{(1 \to b)} \coloneq (s_1, \dots, s_{i-1}, b, s_{i+1}, \dots, s_N)$. Additionally, we define the \textbf{total influence} of $f$ to be
    \[\mathbf{I}[f] = \sum_{i=1}^N \mathbf{Inf}_i [f]. \]
\end{definition}
In the r.h.s. of the inequality in Eq.~\eqref{eq:bound}, when the total influence satisfies \( \mathbf{I}(\partial E_{\boldsymbol{\theta}}/\partial \theta_a) \sim O(N) \), as we found, for instance, in the RBM case, this term bounds the l.h.s., which involves a sum over \( \sum_{k=n}^N \binom{N}{k} \) elements. Additionally, the factor \(1/n\) in the r.h.s. implies that the bound decreases strictly with the order.

The key insight that follows from Eqs.~\eqref{effective_couplings_dynamics_5} and~\eqref{eq:bound} is that gradient descent does not learn all effective interactions uniformly. Proposition~\ref{bound_proposition} shows that the total dynamical weight associated with effective couplings of order at least \(n\) is bounded as in Eq.~\eqref{eq:bound}, implying a progressive suppression of higher-order contributions. As a result, higher-order effective couplings are learned more slowly during training than lower-order ones. This naturally induces a hierarchical learning dynamics in the effective interaction space. Lower-order terms, such as one-body and two-body interactions, are typically learned first, whereas higher-order couplings emerge only at later stages of training. Importantly, this hierarchy is not solely determined by the structure of the data correlations, but is also a direct consequence of the gradient-descent dynamics itself. Even when the data contain correlations across many orders, the dynamics intrinsically favour the learning of the lowest-order effective interactions at early times.

The same hierarchy is reflected in the evolution of the model correlation functions. Indeed,
\begin{equation} \langle \dot{\boldsymbol{O}} \rangle_{\mathrm{model}} = \left( \frac{\partial \langle {\boldsymbol{O}} \rangle_{\mathrm{model} }}{\partial \boldsymbol{\phi}} \right) \dot{\boldsymbol{\phi}} = \mathrm{Cov}_{\mathrm{model}}(\boldsymbol{O}) \, \dot{\boldsymbol{\phi}}, \end{equation}
At early stages of training, the model is typically close to a weakly correlated distribution, for which
$\mathrm{Cov}_{\mathrm{model}}(\boldsymbol{O}) \approx \mathrm{I}$. 
In this regime, changes in the effective parameters are therefore transferred almost directly to the corresponding moments. As a consequence, the low-order moments of the model align with those of the data first, while higher-order correlations are incorporated progressively.

This prioritisation of low-order moments is also important for generalisation. In realistic datasets, the dominant and most reliably estimable statistical structure is often contained in low-order correlations. The dynamics, therefore, act as an implicit regularisation mechanism: they first fit the most robust collective features of the data before introducing higher-order corrections.

Finally, this hierarchy constrains the influence of spurious stationary points on the learning dynamics. Since low-order contributions dominate the gradient norm during the early and intermediate stages of training, cancellation of the total gradient generically requires the corresponding low-order moment mismatches to be simultaneously reduced. As a consequence, stationary points that fail to reproduce the low-order statistics of the data become dynamically difficult to reach. Although the bound does not rigorously exclude the existence of non-data-consistent stationary points, it implies that gradient descent is naturally biased toward trajectories that progressively align the low-order moments of the model with those of the data, thereby approaching the data-consistent fixed point.

\begin{figure}
    \centering
    \includegraphics[width=1\linewidth]{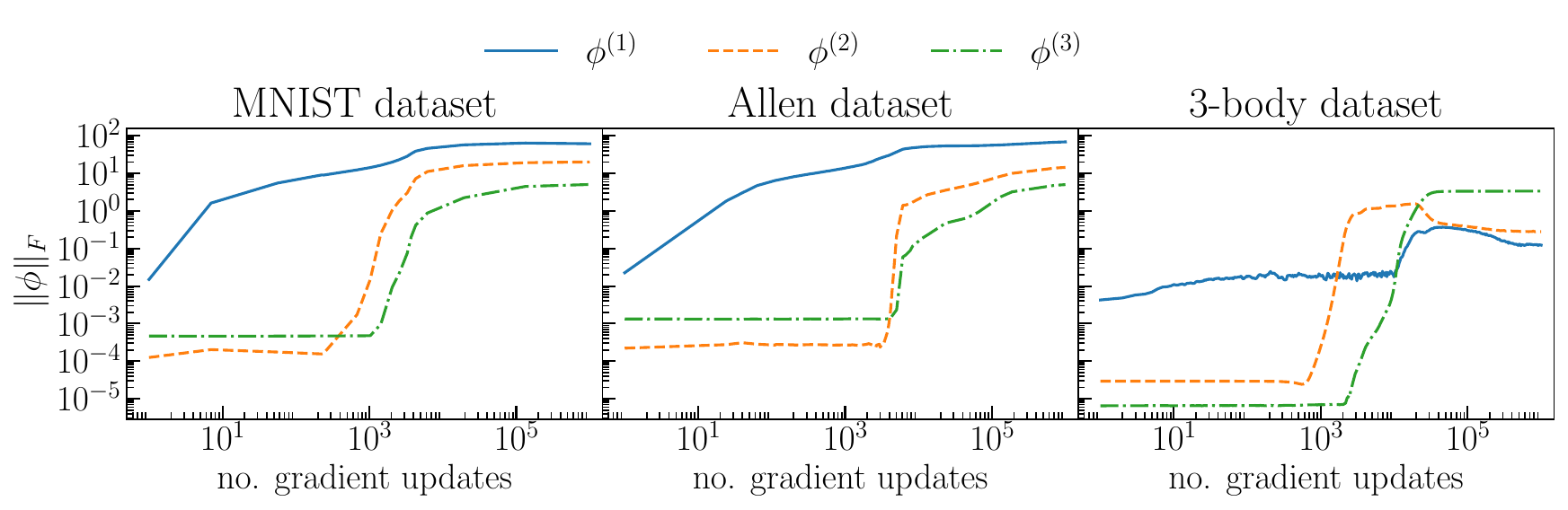}
    \caption{Frobenius norm of the effective parameters, extracted from three distinct RBM trainings using the mapping of Ref.~\cite{decelle2024inferring}, at orders $n=1,2,3$ as a function of the number of gradient updates. A clear sequential learning dynamics is observed: first the fields are learned, followed by the two-body interactions, and finally the three-body interactions.
}
    \label{fig:frobenius}
\end{figure}

\section{Numerical Experiments}
In order to validate and illustrate our theoretical results, we perform simulations with RBMs on three datasets: (i) MNIST~\cite{lecun1989backpropagation} binarised to ${0,1}$ variables, (ii) a dataset of neuronal activity recorded in the Allen Institute Visual Behaviour Neuropixels dataset~\cite{AllenData}, and (iii) a synthetic spin system with Hamiltonian
\begin{equation}
H_{\mathrm{data}}(\boldsymbol{s}) = -\beta \sum_{i=1}^N s_i s_{i+1} s_{i+2},\label{eq:3-spin}
\end{equation}
with $\beta=0.5$, $N=48$, and periodic boundary conditions ($s_{N+1}=s_1$).
All three datasets exhibit non-trivial higher-order structure, so that their empirical distributions cannot be captured by pairwise Boltzmann machines. Details of the training procedure are provided in the Appendix~\ref{Ap:training}.

In Fig.~\ref{fig:frobenius}, we report the Frobenius norm of the $n$-th order effective couplings, computed using the relations from~\cite{decelle2024inferring,decelle2025inferring} (reproduced in the Appendix~\ref{Ap:couplings}), for $n=1,2,3$:
\begin{equation*}
\textstyle\lVert \boldsymbol{\phi}^{(n)} \rVert_F
=
\left(
\sum_{\substack{I \subseteq [N], \\ |I|=n}}
\left|\phi_I^{(n)}\right|^2
\right)^{1/2}.
\end{equation*}
We clearly observe the effect of DSB: the bias terms are learned first, followed by pairwise interactions, and finally, third-order couplings. This provides empirical evidence that interactions are learned sequentially in increasing order $n$.

Even more strikingly, in the third dataset—where the ground-truth model contains only three-body couplings but nevertheless exhibits non-zero covariances (Fig.~\ref{fig:3-coupling}–Left)—the learning dynamics proceed through distinct stages. First, the model develops small effective fields, followed by the emergence of pairwise couplings that partially account for the observed covariance structure. These two-body interactions are subsequently suppressed as the model converges toward the true three-body interactions.

Finally, Fig.~\ref{fig:3-coupling}–Right shows that two independent trainings on the same dataset yield nearly identical effective parameters, both consistent with the ground-truth model. Exact agreement is not expected due to finite-sample effects, the use of mini-batches for gradient estimation, and the stochastic nature of negative-phase sampling.

\begin{figure}
    \centering
    \includegraphics[height=4.5cm]{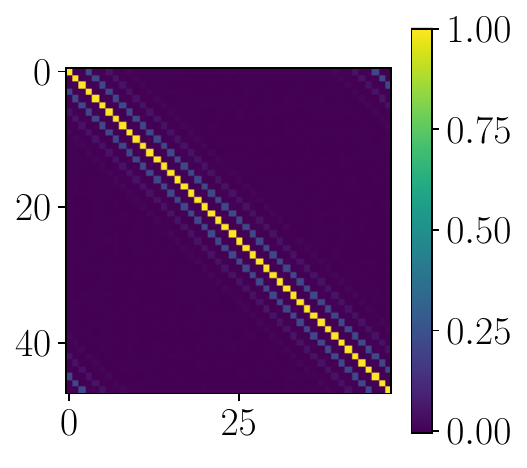}
    \includegraphics[height=4cm]{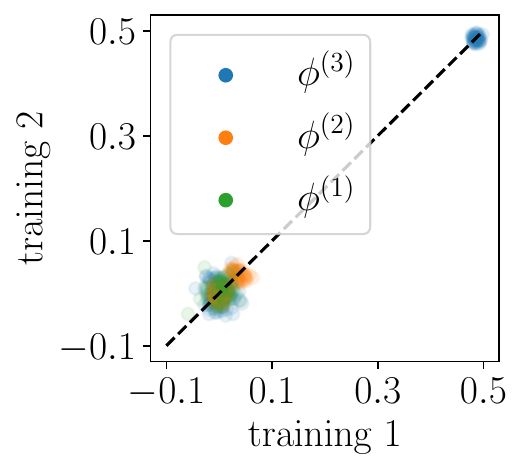}
    \caption{(Left) Covariance matrix of the data generated with the model in Eq.~\eqref{eq:3-spin}. (Right) We show a scatter plot of all $\phi_I^{(n)}$ coefficients inferred at the end of two independent trainings on the same synthetic 3-body dataset. The strong agreement between the two trainings indicates that the learned models are essentially identical and closely match the ground-truth model, which consists solely of a subset of three-body couplings of strength $\beta=0.5$.}
    \label{fig:3-coupling}
\end{figure}

\section{Conclusion}
In this work, we investigated the learning dynamics of EBMs by focusing on the space of effective parameters rather than the weight space. This parametrisation provides a natural framework, as it embeds the model within the exponential family, rendering the likelihood convex in the effective parameters. We showed that likelihood minima that reproduce the empirical statistics are generically marginally stable, but can be further stabilised through $\ell_2$ regularisation of the effective parameters, which can be implemented efficiently. Finally, we argued that the learning dynamics typically exhibit a distributional simplicity bias, which may explain why spurious fixed points of the gradient dynamics are rarely encountered in practice.

Throughout this study, we assumed that the EBM can, in principle, match all moments of the empirical distribution, or equivalently, that the number of model parameters exceeds the number of effective parameters required to describe the ground-truth distribution. While such overparameterised regimes are common in modern machine learning, this assumption is not always satisfied in practice. In the case of RBMs, for instance, computational constraints—most notably the slowing down of sampling dynamics—often limit the feasible number of parameters. It is therefore of considerable interest to understand the extent to which our results extend to more realistic settings in which the model cannot exactly reproduce all moments of the empirical distribution.

\section{Acknowledgment}
The authors acknowledge financial support from grant PID2024-158623NB-C21, funded by MICIU/AEI/10.13039/501100011033 and by ERDF/EU. AJNG acknowledges support from the Comunidad de Madrid under predoctoral contract PIPF-2024/COM-35075.

\bibliographystyle{unsrt}
\bibliography{references}

\appendix

\section{Elements of Pseudo-Boolean function analysis}
\label{pseudo-Boolean_analysis_section}
This appendix collects the theorems of pseudo-Boolean function analysis that underpin the results presented in the main text. Full proofs and additional technical details are deferred to Ref.~\cite{ODonnell_2014}, the standard textbook on the subject.

\textit{Pseudo-Boolean} functions map binary vectors to real numbers:
\[
f \colon \{-1,1\}^N \to \mathbb{R}.
\]
Any energy function defined on binary variables, including the celebrated Ising Hamiltonian,
\[
\mathcal{H}(\boldsymbol{s}) = - \sum_{i<j} J_{ij} s_i s_j - \sum_i h_i s_i,
\]
is a pseudo-Boolean function. The \textit{Fourier Expansion Theorem} of pseudo-Boolean functions implies that any energy function of binary vectors can be expressed as a generalised Ising model with higher-order interactions.

\begin{definition}
\textbf{Parity functions.} For \( I \subseteq \{1, 2, \dots, N \} \coloneq [N]\), we define the parity functions as \( \chi_\emptyset (\boldsymbol{s}) \coloneq 1 \) and \( \chi_I( \boldsymbol{s}) \coloneq \prod_{i \in I} s_i \), if  \( I \neq \emptyset\).
\end{definition}

\begin{theorem}
    \textbf{Fourier Expansion Theorem.} Every function \(f \colon \{ -1, 1\}^N \to \mathbb{R} \) admits a unique expansion in terms of parity functions,
    \begin{equation}
    f(\boldsymbol{s}) = \sum_{I \subseteq [N]} \hat{f}(I) \chi_I( \boldsymbol{s})
    \label{fourier_expansion}
    \end{equation}
    This expression is called the \textit{Fourier expansion} of $f$, where the real number $\hat{f} ({I})$ is the \textit{Fourier coefficient} of $f$ on $I \subseteq [N]$. These Fourier coefficients are given by
\begin{equation}
    \hat{f}(I) = \frac{1}{2^N} \sum_{\boldsymbol{s'} \in \{-1, 1 \}^N} f (\boldsymbol{s}') \chi_I( \boldsymbol{s'})
    \label{fourier_coefficient}
\end{equation}
\label{Fourier_expansion_theorem}
\end{theorem}
\begin{corollary}
\label{Fourier_expansion_corollary}
Consider an energy function \(E_{\boldsymbol{\theta}} \colon \{-1, 1 \}^N \to \mathbb{R} \), then we can apply the following expansion
\begin{align}
    E_{\boldsymbol{\theta}}(\boldsymbol{s}) 
    & \equiv -\sum_{i=1}^N {h_i (\boldsymbol{ \theta}}) \ s_i -\sum_{i<j} J^{(2)}_{ij} (\boldsymbol{\theta}) \ s_i s_j - \sum_{i<j<k} J_{ijk}^{(3)} (\boldsymbol{\theta}) \ s_i s_j s_k + \dots \nonumber \\
    &= \sum_{i=1}^N h_i (\boldsymbol{\theta}) \ s_i +  \sum_{n=2}^N \sum_{i_1 < \dots < i_n} J_{i_1 \dots i_n}^{(n)} (\boldsymbol{\theta}) \ \prod_{k=1}^n s_{i_k}.
    \label{generalized_ising_model}
\end{align}
Where, \( h_{i} (\boldsymbol{\theta}) \coloneq -\hat{E}_{\boldsymbol{\theta}}(\{ i \}) \) and \(  J_{i_1 \dots i_n}^{(n)} (\boldsymbol{\theta}) \coloneq - \hat{E}_{\boldsymbol{\theta}}( \{i_1, \dots, i_n \}). \)  
Note that the term zero order term \(\hat{E}_{\boldsymbol{\theta}} (\emptyset) \)  is missing from the r.h.s. of~\eqref{generalized_ising_model} as it is a constant that does not affect the Boltzmann distribution associated with $E_{\boldsymbol{\theta}}$. 
\end{corollary}
From a linear algebra perspective, the set of all pseudo-Boolean functions forms a $2^N$-dimensional vector space. A natural basis for this space is given by the parity (Fourier) functions, which are linearly independent. In what follows, we introduce an inner product on this space and show that these parity functions in fact form an orthogonal basis of it.
\begin{definition}
\textbf{Inner Product.} Given two functions \( f, g \colon \{-1, 1\}^N \to \mathbb{R} \), we define the inner product on them by
\begin{equation}
    \langle f, g  \rangle = \frac{1}{2^N} \sum_{\boldsymbol{s}' \in \{-1, 1 \}^N }  f(\boldsymbol{s}') g(\boldsymbol{s}').
\end{equation}
\end{definition}
\begin{proposition}
For \( \boldsymbol{s} \in \{-1, 1\}^N \) it holds that \(\chi_I (\boldsymbol{s}) \chi_J (\boldsymbol{s}) = \chi_{I \triangle J} (\boldsymbol{s})\)· Where we have introduce the symmetric difference \( I \triangle J = (I\cup J ) \setminus (I \cap J) \). 
\end{proposition} 
\begin{proposition}
\( 2^{-N} \sum_{\boldsymbol{s} \in \{ -1, 1\}^N} \chi_I (\boldsymbol{s}) = \begin{cases}
    1 & \mathrm{if} \ I = \emptyset, \\
    0 & \mathrm{if} \ I \neq \emptyset.
\end{cases}\)
\end{proposition}
From the above propositions and the definition of the inner product, it follows the orthogonality condition of parity functions.
\begin{theorem}
    \textbf{Orthogonality of parity functions.} The \(2^N\) parity functions satisfy the orthogonality condition \[ \langle \chi_I, \chi_J \rangle = 
    \begin{cases}
        1 & \mathrm{if} \ I=J, \\
        0 & \mathrm{if} \ I \neq J.
    \end{cases}. \]
\end{theorem}
From the orthogonality of the parity functions and the bilinearity of the inner product, it follows the \textit{Plancherel's} and \textit{Parseval's Theorems}.
\begin{theorem}
    \textbf{Plancherel's Theorem.} For any \(f, g \colon \{-1, 1 \}^N \to \mathbb{R}, \)
    \[ \langle f, g \rangle = \sum_{I \subseteq [N] } \hat{f}(I) \hat{g}(I). \]
\end{theorem}
\begin{corollary}
\label{Parseval_theorem}
    \textbf{Parseval's Theorem.} For any \( f \colon \{-1,1 \}^N \to \mathbb{R}, \)
    \[ \langle f, f \rangle = \sum_{I \subseteq [N] } \hat{f}(I)^2. \]
\end{corollary}

A very useful function in pseudo-Boolean function analysis is the \textit{influence} of the $i$-th variable, which measures how sensitive the value of a function \( f \) is to flipping the $i$-th input variable. Additionally, we may also define the \textit{total influence} of $f$ as the sum of the influences over all variables. The formal definition of these quantities is given in the main text (see Definition~\ref{influence_definition}).

The application of the Fourier expansion and Parseval's theorems to the definition of the influence leads automatically to a Fourier formula for the influence:
\begin{theorem}
    For \( f \colon \{-1, 1 \}^N \to \mathbb{R} \) and \( i \in [N], \)
    \[\mathbf{Inf}_i [f] = \sum_{I \ni i} \hat{f}(I)^2. \]
    \label{influence_formula}
\end{theorem}
Similarly, if we replace Theorem~\ref{influence_formula} in the total influence definition, we note that the Fourier weight \(\hat{f}(I)^2\) is counted exactly \(|I|\) times inside the sum over all \( i \in [N] \). Hence:
\begin{theorem}
     For \( f \colon \{-1, 1 \}^N \to \mathbb{R} \),
     \[ \mathbf{I}[f] = \sum_{i=1}^n \sum_{\substack{I \subseteq [N], \\ I\ni i}} \hat{f}(I)^2 = \sum_{n=0}^N n \left[ \sum_{\substack{I \subseteq [N], \\ |I| = n }} \hat{f}(I)^2 \right]. \]
     \label{total_influence_formula}
\end{theorem}
Finally, Theorem~\ref{total_influence_formula} allows us to derive a \textit{spectral concentration theorem}. 
\begin{theorem}
\textbf{Low-degree Spectral Concentration.} For \( f \colon \{-1, 1 \}^N \to \mathbb{R} \) and \( k \in \mathbb{R}, \)
\[ \sum_{\substack{I \subseteq [N], \\ |I| \ge n}} \hat{f}(I)^2 \le \frac{\mathbf{I}[f]}{k}.  \]
\label{low-degree_concentration_theorem}
\end{theorem}

\section{Details on the training of the RBMs used for the numerical experiments} \label{Ap:training}
We consider Bernoulli–Bernoulli RBMs defined through the marginal energy over visible variables, obtained after summing over hidden units:
\begin{equation}
    \mathcal{H}_{\boldsymbol{\theta}}(\boldsymbol{v})
    =
    - \sum_{i=1}^{N_\mathrm{v}} b_i v_i
    - \sum_{a=1}^{N_\mathrm{h}} \ln \left( 1 + e^{c_a + \sum_i W_{ia} v_i} \right),
    \label{eq:energy_visible}
\end{equation}
where $\boldsymbol{\theta} = \{b_i, c_a, W_{ia}\}$ denotes the model parameters, and $N_\mathrm{v}$ and $N_\mathrm{h}$ are the numbers of visible and hidden units, respectively.

We trained such RBMs on the following datasets:
(i) MNIST binarized to $\{0,1\}$ variables~\cite{mnist} ($N_\mathrm{v}=784$, $N_\mathrm{h}=500$ and $M=50000$ samples),
(ii) a dataset of neuronal activity consisting of binarised spike trains from mouse neurons recorded in the Allen dataset~\cite{AllenData}, see~\cite{bereux2026uncovering} for details on its construction, ($N_\mathrm{v}=1569$, $N_\mathrm{h}=128$ and $M=31499$ samples),
and (iii) a synthetic spin system with purely three-body interactions ($N_\mathrm{v}=48$, $N_\mathrm{h}=128$ and $M=700000$ samples).

All RBMs are trained by maximum likelihood. The biases $\boldsymbol{b}$ and $\boldsymbol{c}$ are initialised to zero, while the weights $W_{ia}$ are initialised from a Normal distribution with variance $10^{-4}$. The gradient of the log-likelihood is estimated using Trajectory Parallel Tempering~\cite{bereux2024fast,bereux2025training}, ensuring proper thermalisation of the Markov chains throughout training.

Training is performed with mini-batches of $5000$ samples and the same number of parallel chains. We use an adaptive learning rate and include $\ell_2$ regularisation with coefficient $10^{-4}$.

\subsection{Effective couplings} \label{Ap:couplings}
For the RBM, the authors of~\cite{decelle2024inferring} have derived a set of formulas that relate the RBM's weights to the effective coupling. The formula reads
\begin{equation}
J_{j_1 \dots j_n}^{(n)} \!=\! \frac{1}{2^n} \sum_i \mathbb{E}_{X_i^{(j_1 \dots j_n)}} \!\left[ \sum_{\sigma'_{j_1} = \pm 1}\! \cdots \!\sum_{\sigma'_{j_n} = \pm 1} \sigma'_{j_1}\! \dots\! \sigma'_{j_n} \ln \cosh \left( \sum_{\mu=1}^n w_{i{j_\mu}} \sigma'_{j_\mu} \!+\! X_i^{\left( j_1 \dots j_n \right)} \!+\! \zeta_i \right) \right]. \label{N-body_formula}
\end{equation}
where
\begin{equation} X_i^{(j_1 \dots j_n)} \equiv \sum_{\mu = n + 1 }^{N_\mathrm{v}} w_{i j_\mu} \sigma'_{j_\mu}.\label{def:X_i}
\end{equation}

\end{document}